\documentclass{article}

\usepackage{microtype}
\usepackage{graphicx}
\usepackage{subfigure}
\usepackage{booktabs} 

\usepackage{hyperref}

\usepackage[accepted]{icml2024}

\usepackage{amsmath}
\usepackage{amssymb}
\usepackage{mathtools}
\usepackage{amsthm}

\usepackage[capitalize,noabbrev]{cleveref}

\theoremstyle{plain}
\newtheorem{theorem}{Theorem}[section]
\newtheorem{proposition}[theorem]{Proposition}

\theoremstyle{definition}

\theoremstyle{remark}

\usepackage[textsize=tiny]{todonotes}

\icmltitlerunning{SENSOR: Imitate Third-Person Expert's Behaviors via Active Sensoring}

\begin{document}

\twocolumn[
\icmltitle{SENSOR: Imitate Third-Person Expert's Behaviors via Active Sensoring}



\icmlsetsymbol{equal}{*}

\begin{icmlauthorlist}
\icmlauthor{Kaichen Huang}{equal,yyy,comp}
\icmlauthor{Minghao Shao}{equal,yyy,comp}
\icmlauthor{Shenghua Wan}{yyy,comp}
\icmlauthor{Hai-Hang Sun}{yyy,comp}
\icmlauthor{Shuai Feng}{yyy}
\icmlauthor{Le Gan}{yyy,comp}
\icmlauthor{De-Chuan Zhan}{yyy,comp}
\end{icmlauthorlist}

\icmlaffiliation{yyy}{National Key Laboratory for Novel Software Technology, Nanjing University, China
}
\icmlaffiliation{comp}{School of Artificial Intelligence, Nanjing University, China}

\icmlcorrespondingauthor{Kaichen Huang}{huangkc@lamda.nju.edu.cn}
\icmlcorrespondingauthor{Minghao Shao}{shaomh@lamda.nju.edu.cn}

\icmlkeywords{Machine Learning, ICML}

\vskip 0.3in
]




\begin{abstract}
In many real-world visual Imitation Learning (IL) scenarios, there is a misalignment between the agent's and the expert's perspectives, which might lead to the failure of imitation. Previous methods have generally solved this problem by domain alignment, which incurs extra computation and storage costs, and these methods fail to handle the \textit{hard cases} where the viewpoint gap is too large. To alleviate the above problems, we introduce active sensoring in the visual IL setting and propose a model-based \textbf{SENS}ory imitat\textbf{OR} (SENSOR) to automatically change the agent's perspective to match the expert's. SENSOR jointly learns a world model to capture the dynamics of latent states, a sensor policy to control the camera, and a motor policy to control the agent. Experiments on visual locomotion tasks show that SENSOR can efficiently simulate the expert's perspective and strategy, and outperforms most baseline methods.
\end{abstract}

\section{Introduction and related works}

Visual imitation learning (IL) \cite{lopes2005visual,finn2017one,li2017infogail,tamar2018imitation,rafailov2021visual} deals with situations where agents make decisions based on visual observation (images or videos) rather than true states \cite{xu2017end,sallab2017deep,jaderberg2019human,schrittwieser2020mastering}, and utilize off-the-shelf expert visual demonstration instead of elaborate reward functions. Visual demonstration is a more natural resource for imitation in real-world scenarios, but good performance requires that the agent and expert share a common viewing perspective. In practice, a mismatch often occurs in perspectives when the cameras on the demonstrator and the imitator have different types, positions, or tuning parameters. This comes to a third-person imitation learning setting \cite{shang2021self,garello2022towards,klein2023active} where observations obtained from the demonstrations are not the same as what the imitator agent will be faced with. Manually tuning the agent's camera is non-trivial and requires expensive expertise, so it's crucial to eliminate the mismatch automatically.

\begin{figure}[t]
\centering
\includegraphics[width=0.5\textwidth]{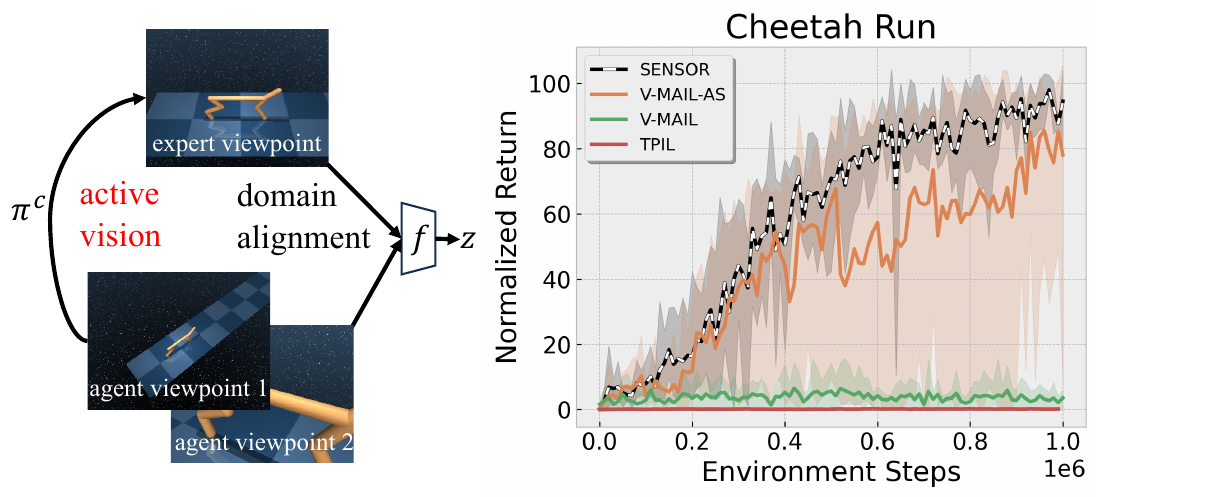}

\caption{\textbf{Left}: Intuitive distinction between observations rendered with different viewpoints. The top is the expert viewpoint, while the bottom is a poorer agent viewpoint. Domain adaptation methods learn an encoder $f$ to map different observations to the same embedding $z$, while active vision explicitly adjusts the agent's viewpoint by taking a sensor policy $\pi^c$. \textbf{Right}: We evaluate SENSOR and other rival methods over three seeds in Cheetah Run under two hard initial perspectives shown in the left. We report the mean (solid) and standard deviation (shaded) of normalized return. SENSOR beats other methods on both performance and stability levels.}
\label{f1}
\end{figure}

Domain alignment \cite{kim2019cross,kim2020domain,raychaudhuri2021cross,franzmeyer2022learn,yu2018one} is a classic way to tackle the gap between perspectives in third-person IL. Several third-person IL methods combines adversarial Imitation Learning along with methods to align expert and agent observation domain. TPIL \cite{stadie2017third} borrows idea from domain confusion and generative adversarial networks (GANs) \cite{goodfellow2014generative}, which uses a discriminator to distinguish between the embedding of expert and agent observation to align them into one subspace. Hierarchical Controller \cite{sharma2019third} generates subgoals on agent domain from expert domain through a GAN-like module, and requires expert demonstrations from both perspective domains to perform alignment. DisentanGAIL \cite{cetin2021domain} uses a mutual information constraint to prevent the policy discriminator to exclusively rely on domain information and requires extra random demonstrations for additional learning signal. Although these methods align views with different ways, they bring extra computation and storage costs and fail to fix the over-large gap between different viewpoints. 

This raises an interesting question: can we design an algorithm that let the agent \textit{\textbf{adjust its perspective automatically to match the expert's}}, thus solving the hard cases of large perspective gap \textit{\textbf{without viewpoint knowledge}}? We try to answer this question through active vision (also called active sensoring). Active vision methods \cite{bajcsy1988active,chen2011active,cheng2018reinforcement,fujita2020distributed,shang2023active,ActiveS2} in visual reinforcement learning get better environmental information by controlling the viewpoint of the camera. The idea of applying active vision is natural. Humans accomplish tasks through adopting motor behaviors as well as adapting their viewpoints. Similarly in IL, humans can learn by observing examples of experts with different perspectives and actively adopting the appropriate perspective to accomplish the task. Bring active vision into visual IL setting shows the ability to align the agent's viewpoint consistent with the expert's, thus reducing third-person IL problem into a simple IL case. Under the activate vision IL framework, we propose a model-based SENSory imitatOR (SENSOR) algorithm, which explicitly models the perspective as learnable parameters and designs a sensor policy to adjust the camera until the agent's viewpoint aligns with the expert's.

Our main contributions are summarized as follows:

\textbf{1)} $\,\,$ We provide insights into understanding domain alignment methods by quantifying the task's difficulty with mutual information. \\
\textbf{2)} $\,\,$ To the best of our knowledge, we are the first to introduce active sensoring in the visual IL setting to tackle IL problems from different viewpoints. \\
\textbf{3)} $\,\,$ We propose model-based \textbf{SENS}ory imitat\textbf{OR} (SENSOR) that bridge the viewpoint gap through active sensoring, and we theoretically analyze the limitations of decoupled dynamics. \\
\textbf{4)} $\,\,$ We conduct extensive locomotion experiments on DMC Suite \cite{tassa2018deepmind} with different perspectives. SENSOR produces excellent results, demonstrating the significance of active sensoring.

\section{Preliminaries}
\label{s-preliminary}

\subsection{Active Vision Framework}
We use the camera as the sensor and accept visual observations as input, thus we adopt the Partially-Observed Markov Decision Process (POMDP) setting, which can be described with the tuple: $\mathcal{M}=(\mathcal{S},\mathcal{A},\mathcal{O},\mathcal{P},r,p_0,\gamma,\phi)$, where $\mathcal{S}$ is the state space, $\mathcal{A}$ is the action space, $\mathcal{O}$ is the observation space, $\mathcal{P}\colon\mathcal{S}\times\mathcal{A}\rightarrow\mathcal{S}$ is the transition function, $r\colon\mathcal{S}\times\mathcal{A}\times\mathcal{S}\rightarrow\mathbb{R}$ is the reward function, $\gamma\in[0,1]$ is the discount factor and $p_0$ is the initial state distribution. Observations are generated through the emission function $o\sim\phi(o|s)$. 

The active vision framework allows agent to autonomously adjust its viewpoint by taking sensory actions, requiring us to define some perspective-related elements. We introduce a set of supplementary settings based on POMDP, which can be described with the tuple: $\mathcal{M}_{\text{supp}}=(\mathcal{Z},\mathcal{C},\mathcal{A}^z,\mathcal{A}^c,f_a)$, where $\mathcal{Z}$ is the motor state space, $\mathcal{C}$ is the sensor state space, $\mathcal{A}^z$ is the motor action space, $\mathcal{A}^c$ is the sensor action space, and $f_a\colon\mathcal{A}^z\times\mathcal{A}^c\rightarrow\mathcal{A}$. Based on $\mathcal{A}^z$, which is the original action space in $\mathcal{M}$, we introduce the sensor action space $\mathcal{A}^z$ and redefine $\mathcal{A}$ as the union of $\mathcal{A}^z$ and $\mathcal{A}^c$ using $f_a$. Similarly, we believe that state $s\in\mathcal{S}$ contains both the motor and sensor information, thus we assume that $\mathcal{S}=\mathcal{Z}\times\mathcal{C}$.

Given a fixed viewpoint $c^*$, we derive a subspace $\mathcal{O}^*$ of observation through $\mathcal{S}^*=\mathcal{Z}\times\{c^*\}$ and emission function. In our assumption the expert holds a fixed viewpoint $c^e$, and the agent has a viewpoint $c^a$ as initialization. Then we define the expert observation subspace $\mathcal{O}^e$ and the agent initial observation subspace $\mathcal{O}^a$:
\begin{align}
    \mathcal{O}^e &= \{o^e | o^e\sim \phi(o|s^e) = \phi(o|z,c^e) \} \\
    \mathcal{O}^a &= \{o^a | o^a\sim \phi(o|s^a) = \phi(o|z,c^a) \}
\end{align}
where $z$ is sampled from the fixed $\mathcal{Z}$.

\subsection{Adversarial Imitation Learning}
\label{s-AIL}
Adversarial Imitation Learning(AIL) \cite{baram2016model,ho2016generative,AIL2,sun2019adversarial,zolna2021task} comes from inverse RL\cite{ng2000algorithms,abbeel2004apprenticeship,MaxEntIRL} and achieves good performance in IL. AIL methods train a discriminator to distinguish agent trajectories and expert demonstrations instead of explicitly deriving a pseudo reward function. AIL policy $\pi$ utilizes signal from discriminator $D$ to minimize the divergence between the expert and agent occupancy measures \cite{AILmatch}. One of classical approaches is GAIL \cite{ho2016generative}, which uses Jensen-Shannon divergence as occupancy measurement. The objective of GAIL is as follows:
\begin{equation}
\label{e-gail}
\begin{split}
    \max_{\pi}\min_{D_{\psi}} \;&\mathbb{E}_{(s, a)\sim \rho^E_{\mathcal{M}}}\Big[-\log D_{\psi}(s, a) \Big]\\
   & + \mathbb{E}_{(s, a)\sim \rho^{\pi}_{\mathcal{M}}}\Big[-\log(1-D_{\psi}(s, a))\Big]
\end{split}
\end{equation}

\section{Domain alignment for third-person Imitation Learning}
\label{s-limitation}
In the general third-person IL setting, the expert and the agent have different perspectives, and the corresponding observation spaces $\mathcal{O}^e$ and $\mathcal{O}^a$. The domain alignment framework maps two observation spaces to the same latent space $\mathcal{Z}$ and trains the agent on the consistent representations. We summarize some techniques commonly used in domain alignment methods and analyze their limitations:

\subsection{Limitation of domain alignment methods}

\textbf{Domain confusion.}\,\, \cite{tzeng2014deep,ganin2015unsupervised,stadie2017third} proposed to learn representations that should be helpful to downstream tasks, but uninformative about the domain of the input \cite{ben2006analysis}. Taking TPIL \cite{stadie2017third} as an example, which formulates the problem as:
\begin{align*}
    \max_{\pi}\min_{\mathcal{D}_R} \mathcal{L}_R = &\sum_i CE(\mathcal{D}_R(\mathcal{D}_F(o_i)),c_{l_i}) \\
    & \text{s.t.}\,\, \text{MI}(\mathcal{D}_F(o);d_l)=0
\end{align*}
where MI is mutual information and $\mathcal{D}_R$, $\mathcal{D}_F$, $c_{l_i}$ and $d_l$ denote the classifier, the feature extractor, the correct class label and the domain label respectively. TPIL attempts to minimize classification loss while maximizing domain confusion using a gradient reversal layer from \cite{ganin2015unsupervised}.

\textbf{Introducing prior data.}\,\, Using extra collected prior datasets to assist training is a common practice in domain adaptation methods. DisentanGAIL \cite{cetin2021domain} collects prior datasets $B_{P.E}$ and $B_{P.\pi}$ by executing random behavior in both expert and agent domains, which are used to constrain the mutual information term:
\begin{align*}
    MI(z_i,d_i|B_{P.E}\cup B_{P.\pi})=0
\end{align*}
where $z_i$ is the state embedding and $d_i=1_{o_i\in B_E\cup B_{P.E}}$ is the binary domain label, TPIL also has the $B_{P.E}$ design.

The information contained in the state embedding can be divided into the \textit{domain information} and the \textit{goal-completion information} \cite{cetin2021domain}. The process of reducing the distance between two distributions in the task-relevant latent space is essentially the process of removing domain information. Removing domain information via maximizing the domain confusion assists the two distributions in becoming indistinguishable. However, these methods have inherent drawbacks. When the gap between $\mathcal{O}^a$ and $\mathcal{O}^e$ is sufficiently large, or the domain contains very little information (e.g., a large distance), it becomes difficult to completely remove domain information from the embedding, which can lead to performance degradation of the algorithm. We attempt to measure the gap between two observation spaces and the amount of information contained in domain by mutual information in \cref{domain-limitation}.

Prior datasets can accelerate the training of domain adaptation methods, but they introduce extra costs. Moreover, agent-environment interaction is expensive in most real-world scenarios, making it impractical to collect a prior dataset in the expert domain. In contrast, SENSOR does not require the participation of prior datasets, thus avoiding the time and space overhead of collecting and storing prior datasets.


\begin{figure}[t]
    \centering
    \includegraphics[width=0.22\textwidth]{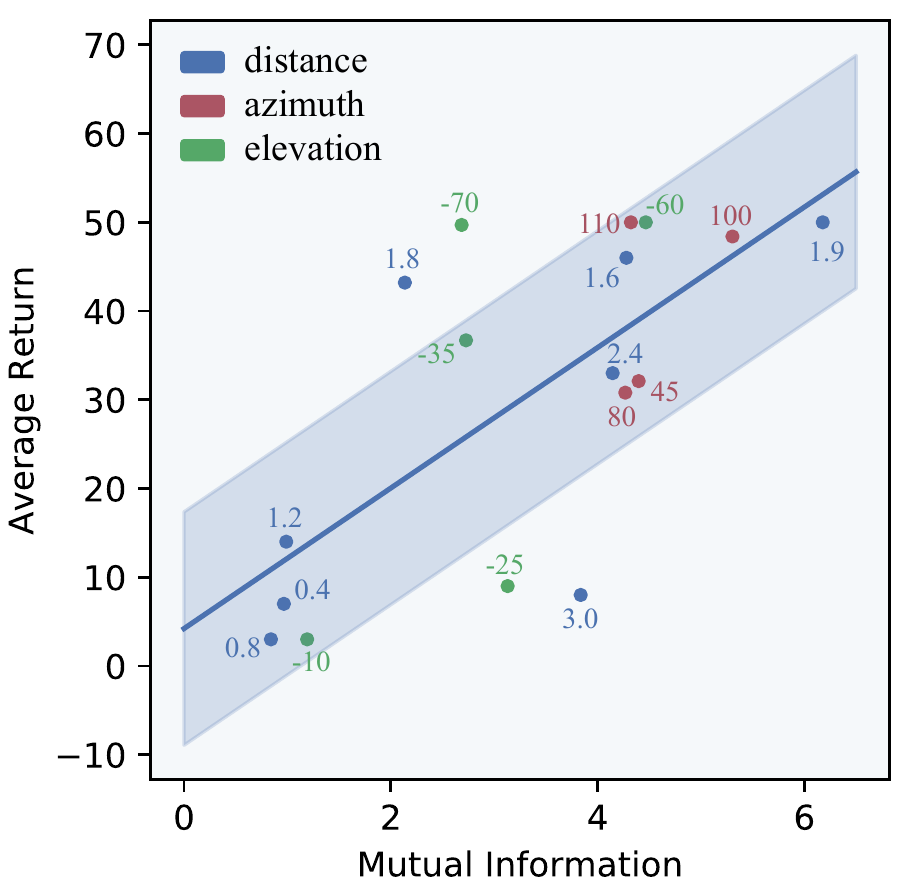}
    \includegraphics[width=0.24\textwidth]{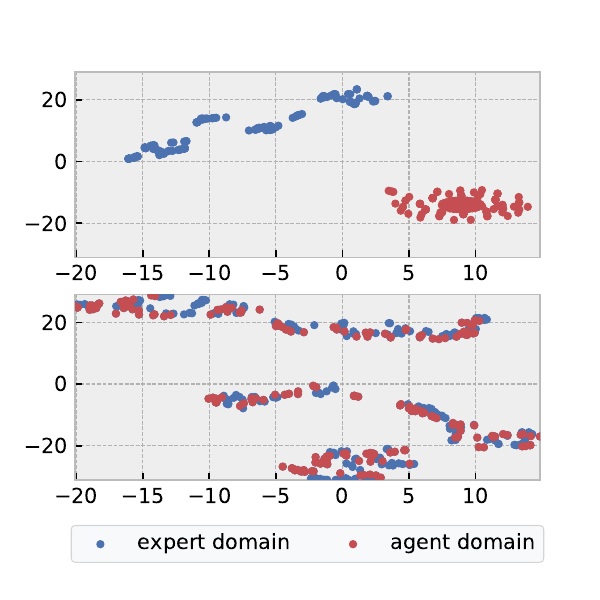}\\
    
    \makebox[0.22\textwidth]{\footnotesize (a)}
    \makebox[0.24\textwidth]{\footnotesize (b)}
    \caption{Results that explain the limitations of domain alignment methods.  \textbf{(a)}: The relationship between mutual information of different viewpoints and performance among different viewpoint settings learned by DisentanGAIL. Each point represents a single viewpoint which differs from the expert's on the label text around it. Details in \cref{domain-limitation}. \textbf{(b)}: The t-SNE plot\cite{van2008visualizing} shows the difference of embedding on two domains learned by two agents which are trained on different viewpoints. Top is a viewpoint far from the expert's and below is a closer one.} 
    \label{limitation}
\end{figure}

\subsection{Revealing the relationship between performance and perspective inconsistency}
\label{domain-limitation}

\textbf{Mutual Information between Observations.}\,\, We attempt to investigate the relationship between the performance of domain adaptation methods and MI through experiments with different initial viewpoints. For simplicity, we only consider one parameter changes shown in \cref{limitation}-(a) and keep the other parameters the same as the expert's. We conduct experiments on the InvertedPendulum environment in MuJoCo using DisentanGAIL for 10k steps. Evaluating the impact of perspective on MI should eliminate the impact of policy. To achieve this, we collect 10 test trajectories $\{\tau^a_i=o^a_{1:H}\}_{i=1}^{10}$ , extract agent states from $\tau^a_{1:10}$ and render new images to compose trajectories $\tau^e_{1:10}$ under the expert perspective. Then we estimate the MI using the MINE method \cite{belghazi2018mutual} on $\tau^a_{1:10}$ and $\tau^e_{1:10}$.

As shown in the correlation plot (\cref{limitation}-(a).), The magnitude of MI is positively correlated with viewpoint difference between agent and expert ($d=1.9$ has the largest MI, while $d=0.8$ has the smallest), but negatively correlated with episodic return. The hard cases (poor agent viewpoint initialization) significantly deviate from the expert perspective, resulting in low MI between the corresponding observation spaces $\mathcal{O}^e$ and $\mathcal{O}^a$, which hinders the learning process of DisentanGAIL.

\textbf{Representation alignment.} Domain alignment approaches aim to map the expert and agent domains into the same latent space, with the alignment degree directly determining the final performance. To further corroborate the aforementioned MI results, we employ DisentanGAIL to learn under two different agent perspectives (near expert perspective \texttt{d1}.\texttt{9a90e-45} and poor initialization \texttt{d0}.\texttt{4a90e-45}) and visualize the learned embeddings using a t-SNE plot. \cref{limitation}-(b) shows that when the initial viewpoint $c^a$ is closet to the expert viewpoint $c^e$, the two distributions are almost indistinguishable, indicating that $\mathcal{O}^a$ and $\mathcal{O}^e$ are mapped to the same hidden space by domain alignment. When the gap between perspectives is large, obtaining a consistent latent space is difficult, which vividly demonstrates why domain alignment methods cannot handle the hard cases.

\begin{figure*}[t] 
    \centering
    \includegraphics[width=\textwidth]{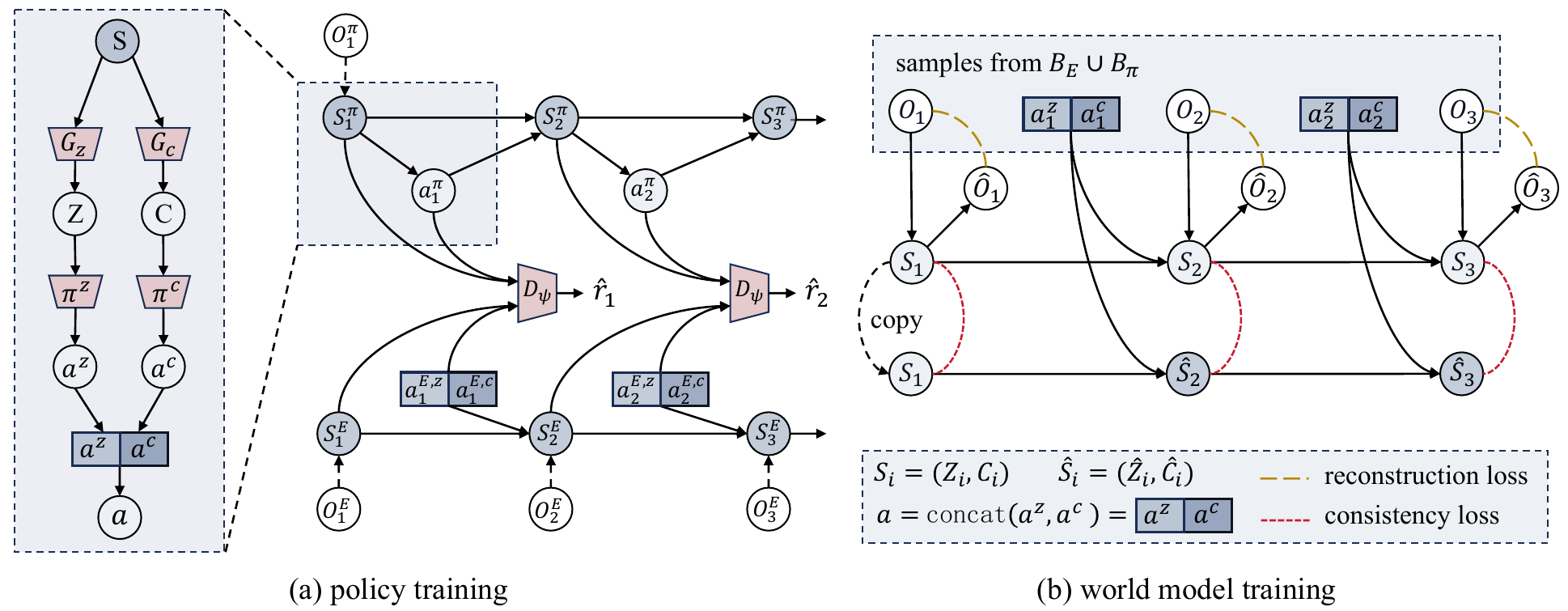}
    \caption{The main framework of the SENSOR method. \textbf{Left}: We use the Recurrent State Space Model(RSSM) \cite{planet} structure to capture the transitions of the latent state $s$, and design two encoders $G_z$ and $G_c$ to extract the motor state $z$ and the sensor state $c$ from $s$. We propose two policy networks $\pi^z$ and $\pi^c$ to make decisions based on $z$ and $c$ respectively, and discriminators $D_{\psi}$ to provide reward signals for the actor updating. Additionally, we feed the concatenation $a=\text{concat}(a^z,a^c)$ into the dynamics model $p_{\theta}$ and the state encoder $q_{\omega}$ to compute the prior and the posterior of state. \textbf{Right}: We apply $p_{\theta}$ and $q_{\omega}$ to compute the prior and the posterior of batch data sampled from $\mathcal{B}_E\cup\mathcal{B}_{\pi}$, and then we update the world model by minimizing the consistency loss $\mathcal{L}_c$ and the reconstruction loss $\mathcal{L}_r$ mentioned in \cref{s-components}.} 
    \label{f2} 
\end{figure*}

\section{Method}

In this section we demonstrate the SENSOR algorithm to tackle visual imitation learning task under the active vision framework. SENSOR proposes three designs, which will be presented next. In \cref{s-ensemble}, an ensemble method is used to stablize learning process of discriminator. In \cref{s-reward}, we integrate variance information into reward for better iteration with large observation space. Finally, in \cref{s-components}, we describe how SENSOR uses separate motor and sensor polices, rather than a combined single policy, and applies model-based visual imitation learning within an active vision framework.

\subsection{Discriminator Ensemble (DE)}
\label{s-ensemble}

Adopting viewpoint actions will make the observation space more varied compared to the original motor actions. Under the model-based AIL framework, the discriminator, which accepts the output of the encoder and provides the pseudo reward for the policy, is the key module of AIL, and is more likely to perform instability in the face of observation changes (see ablation study \cref{s-exp-ablation}). Motivated by ensemble methods and dropout technique\cite{dropout}, we design a set of $N_1$ discriminators $D_{\psi}=\{D_{\psi}^1,...,D_{\psi}^{N_1}\}$. When updating the discriminators, we select $N_2$ discriminators and drop the others. We use integrated output of selected ones $D_{\psi}^{\text{train}}=\{D_{\psi}^{i_1},...,D_{\psi}^{i_{N_2}}\}$ as the reward value:
\begin{align}
    r_t &= f_{\epsilon}(D_{\psi}^{i_1}(s_t,a_t),...,D_{\psi}^{i_{N_2}}(s_t,a_t)) \\
    &= f_{\epsilon}(r_t^{i_1},...,r_t^{i_{N_2}})
\end{align}

where $f_{\epsilon}$ denotes the $\epsilon$-reward function to be introduced subsequently. The drop operation avoids co-adaptation of different discriminators, thus enhancing the diversity for better ensemble performance. When training the actor, we use the integrated results of all $N_1$ discriminators $D_{\psi}(s_t,a_t)$ as the output.

\subsection{$\epsilon$-Reward}
\label{s-reward}
The alignment speed is an important evaluation metric in time-constrained scenarios, thus we require the agent to align its perspective as quickly as possible. In the beginning, exploration of perspectives should be encouraged to swiftly escape poor initial perspective and avoid getting trapped in local optima. As $c^a$ gradually approaches $c^e$, exploitation of the current perspective should be incentivized to prevent excessive fluctuations. In addition, we utilize the variance information of ensemble discriminators to automatically identify situations lacking exploration. We design the $\epsilon$-reward to adjust the trade-off between exploration and exploitation adaptively:
\begin{align}
    r_t = D_{\psi}^{\text{train}}(s_t,a_t) &= f_{\epsilon}(r_t^{i_1},...,r_t^{i_{N_2}}) \\
    &= \mu + \epsilon\sigma
\end{align}
where $\mu$ and $\sigma$ are mean and standard deviation of rewards $\{r_t^{i_1},...,r_t^{i_{D_2}}\}$ and $\epsilon$ is the trade-off weight. We set $\epsilon$ to linearly decrease from $\epsilon_0$ at the first epoch to $-\epsilon_0$ at the last epoch, where $\epsilon_0$ is a hyperparameter. We empirically show that $0.1$ is suitable for $\epsilon_0$ in \cref{s-exp-ablation}.

\subsection{Components of SENSOR}
\label{s-components}
\textbf{Model Learning.}\,\, We design the following dynamics model $\hat{s}_t\sim p_{\theta}(\hat{s}_t|\hat{s}_{t-1},a_{t-1})$ to compute the prior of state from $\hat{s}_{t-1}$ and $a_{t-1}$. After getting observation $o_t$, we can obtain a more accurate posterior estimate using state encoder $s_t\sim q_{\omega}(s_t|s_{t-1},a_{t-1},o_t)$. Additionally, We design the observation model $\hat{o}_t\sim p_{\zeta}(\hat{o}_t|s_t)$ to reconstruct observation based on $s_t$. We expect the world model to make the prior as close to the posterior as possible, and $s$ should contain enough information to reconstruct the entire image, then we obtain the following Evidence Lower Bound (ELBO):
\begin{equation}
    \label{e3}
    \mathcal{L} \doteq \mathbb{E}_p (\sum_t( \mathcal{L}^t_r + \mathcal{L}^t_c )) + \text{const} 
\end{equation}
where $\mathcal{L}^t_c \doteq \beta\text{KL}(p_{\theta}(\hat{s}_t|\hat{s}_{t-1},a_{t-1})\Vert q_{\omega}(s_t|s_{t-1},a_{t-1},o_t))$ is the consistency loss, and $\mathcal{L}^t_r \doteq -\ln q_{\zeta}(o_t|s_t)$ is the reconstruction loss. 

\textbf{Policy Learning.}\,\, We design a set of discriminators $D_{\psi}=\{D_{\psi}^1,...,D_{\psi}^{N_1}\}$ to distinguish whether the $(s,a)$ pair come from expert buffer $\mathcal{B}_E$, and use its outputs as reward signals while training the motor policy and the sensor policy. The specific usage of this set of discriminators is introduced in \cref{s-ensemble}, and we simply denote it as $r_t=D_{\psi}(s_t,a_t)$ for now. The optimization objective is the same as \cref{e-gail} ($\pi_e$ is the expert policy):

To separate the motor information and the sensor information from state $s_t$, we designed two encoders $G_z$ and $G_c$. We have $z_t=G_z(s_t)$ and $c_t=G_c(s_t)$, where $z_t$ is the motor state and $c_t$ is the sensor state. Additionally, we design two independent policies: the motor policy $\pi^z$ takes $z$ as input and outputs the motor action $a^z\sim\pi^z(a^z|z)$ to control the agent; the sensor policy $\pi^c$ takes $c$ as input and creates the sensor action $a^c\sim\pi^c(a^c|c)$ to adjust the parameters of the camera. We use the concatenation $a=\text{concat}(a^z,a^c)$ to model $f_a$ mentioned in \cref{s-preliminary}. 

We collect imaginary trajectories in the world model and use them to train the motor policy and the sensor policy. Collecting data in imagination rather than in interactions with the environment can improve sample efficiency and thus fuel the training process. Inspired by \cite{hafner2019dream}, we adopt the Actor-Critic method to train two policies. The actor is $a_t=\text{concat}(a_t^z,a_t^c)$, where $a_t^z\sim\pi^z(\cdot|G_z(s_t))$ and $a_t^c\sim\pi^c(\cdot|G_c(s_t))$. We compute critic target (return) $v_t^K$ by $v_t^K = v^K(s_t) = \sum^{t+K-1}_{\tau=t}\gamma^{\tau-t}\log r_{\tau} + \gamma^K v_{t+K}$, where $r_{\tau}=D_{\psi}(s_{\tau}, a_{\tau})$ is the reward, $v_{t+K}=V(s_{t+K})$ is the state value computed by critic $V$. We update the critic by minimizing the gap between $V(s_t)$ and $v_t^K$, and update the actor by maximizing the expected return. The detailed update process is shown in \cref{alg1}.

\begin{figure*}[t] 
    \centering
    \includegraphics[width=\textwidth]{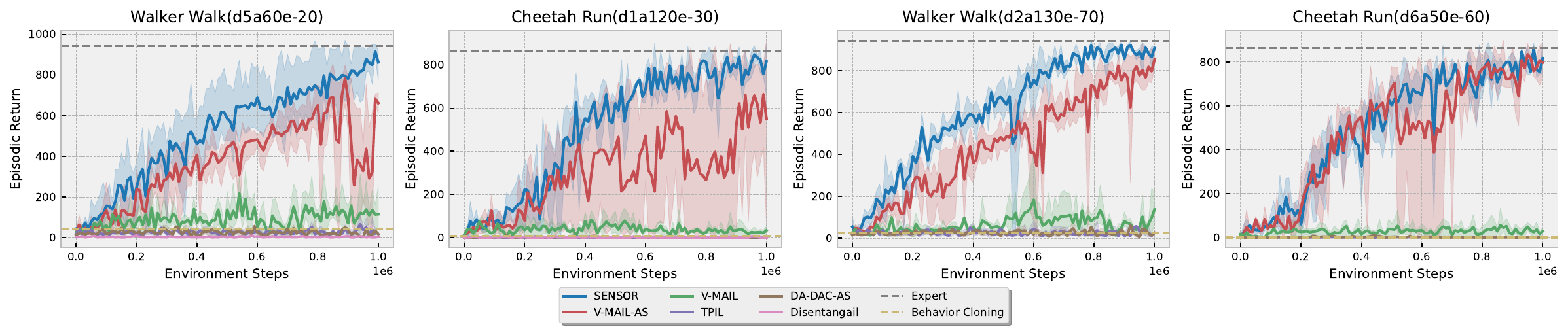}
    \caption{Evaluation results of SENSOR and other baseline methods over three seeds in two visual control tasks in DMC for $1$M steps with different initializations of agent's perspective. The specific settings of the environment and the initial viewpoint are shown above each figure. The solid lines represent the average episodic returns, and the shaded areas around them represent the variance of the performances on different seeds. The gray dotted line denotes the return of expert policy, while the yellow dotted line is the performance of Behavior Cloning. SENSOR outperforms other methods in terms of both performance and stability under different views in different environments.} 
    \label{exp-main}
\end{figure*}

\section{Experiments}
\label{exp}
We conduct several experiments to answer the following scientific questions: (1) How effective is SENSOR in third-person imitation learning tasks? (2) How is the contribution of different components of SENSOR to its final performance? (3) Is the performance of SENSOR consistent and robust to changes in perspective? (4) How does decoupled-dynamics method perform compared to SENSOR?

\textbf{Locomotion Tasks.}\,\, We test the ability of SENSOR and other baseline methods with two visual control tasks from Deep-Mind Control Suite \cite{tassa2018deepmind}: \textit{Walker Walk} and \textit{Cheetah Run}. These tasks are challenging due to high-dimension observation space and contact dynamics. The input observation images have the same shape $64\times 64\times 3$.

\begin{figure}[t] 
    \centering
    \includegraphics[width=0.48\textwidth]{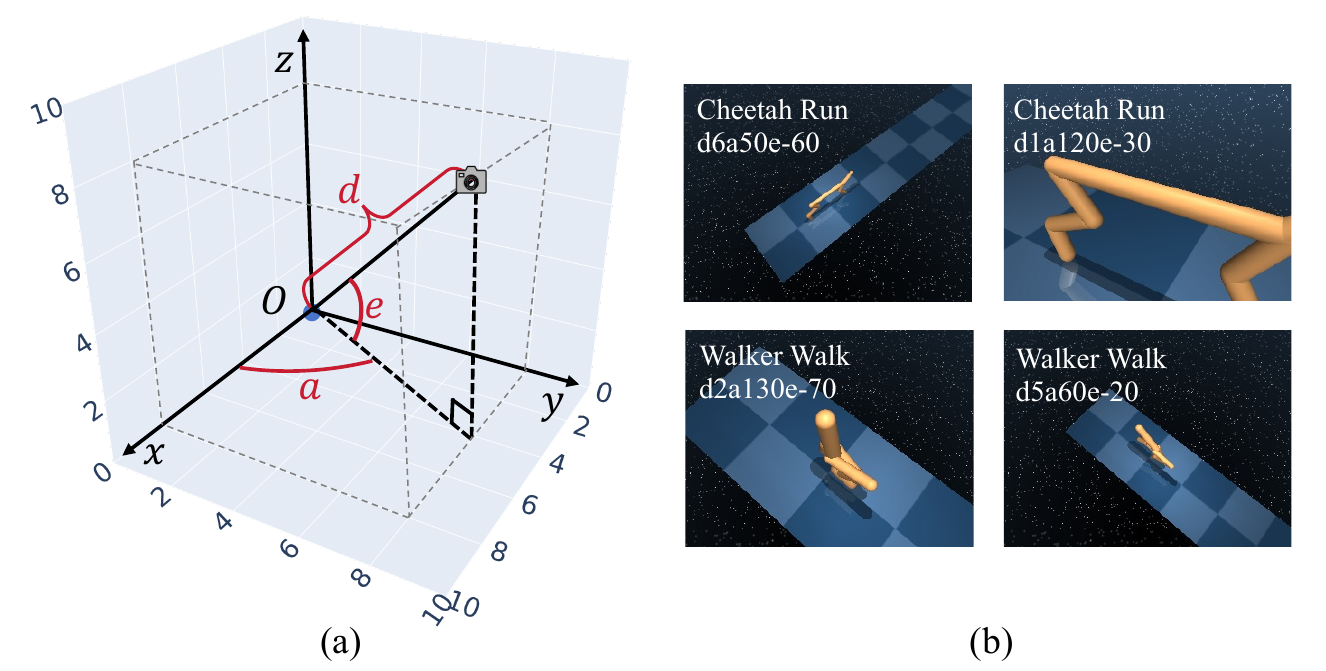}
    \caption{\textbf{(a)}: Camera parameters can be represented as a tuple $(d,a,e)$, where $d$ denotes distance from camera to target point $O$, $a$ is the horizontal angle relative to $O$ and $e$ is the vertical angle relative to $O$. Detailed perspective specifications in \cref{exp}. \textbf{(b)}: Selected hard viewpoints for latter experiments.} 
    \label{four-views}
\end{figure}

\textbf{Specification of perspective.} \,\, As shown in \cref{f2}, each sensor state $c$ has three elements $(d,a,e)$, where $d$ is the distance from the camera to a target point $O$ which can also be referred to as \textit{look at}, $a$ is azimuth angle (horizontal angle relative to $O$) and $e$ is elevation angle (vertical angle relative to $O$). These variables have ranges of $d\in[0,10]$, $a\in[0,180]$, and $e\in[-90,0]$. The expert demonstration has fixed sensor state $(3, 90, -45)$ which is unknown to agent. Sensor action $a^c$ has same shape as $c$, and we set $a^c$ as the additive correction to the sensor state. All three dimensions of $a^c$ have the same range $[-1, 1]$.

\textbf{Baselines.}\,\, Behavior Cloning (\textbf{BC}) takes a supervised learning way to learn a policy that directly maps states to actions using expert demonstrations. The remaining methods were divided into model-based (MB) and model-free (MF) categories. MB methods like \textbf{V-MAIL} \cite{rafailov2021visual} is a variational model-based adversarial imitation learning algorithm, while \textbf{V-MAIL-AS} (AS denotes Active Sensoring) incorporates V-MAIL with active vision by extending the action space $\textit{concat}(a^z,a^c)=a\sim\pi(a|s)$. MF methods like \textbf{DA-DAC-AS} is a variant of DA-DAC, which is used in \cite{wan2023semail}, by extending the action space similarly to V-MAIL-AS, while \textbf{TPIL} \cite{stadie2017third} and \textbf{DisentanGAIL} 
\cite{cetin2021domain} are introduced in \cref{s-limitation}.

\subsection{How effective is SENSOR in third-person imitation learning tasks?}
\label{s-exp-performance}
As shown in \cref{exp-main}, SENSOR surpasses all other baseline methods in both final performance and training stability, suggesting the validity of adjusting perspective actively. It is worth mentioning that besides observation, DisentanGAIL also requires true states as additional input. However, DisentanGAIL and TPIL fail in all experiments, indicating that domain adaptation cannot solve these hard cases with large viewpoint gaps, further verifying the conclusion we obtained in \cref{domain-limitation}. V-MAIL outstrips DA-DAC-AS, disentanGAIL and TPIL slightly, which demonstrates the superiority of MB methods. We believe that a well-trained world model (such as RSSM) naturally has the ability to combat distributional shifts. V-MAIL-AS builds on V-MAIL by giving the agent the ability to actively change its viewpoint. The difficulty of the task decreases as the viewpoint gets closer, and V-MAIL-AS's performance is significantly improved compared to V-MAIL's. DA-DAC-AS is an MF method that performs poorly even with active vision. By comparing it with V-MAIL-AS, we conclude that the world model is the main factor that enables active vision to work, and robust dynamics can facilitate the learning of sensor policy.

\begin{figure}[t] 
    \centering
    \includegraphics[width=0.48\textwidth]{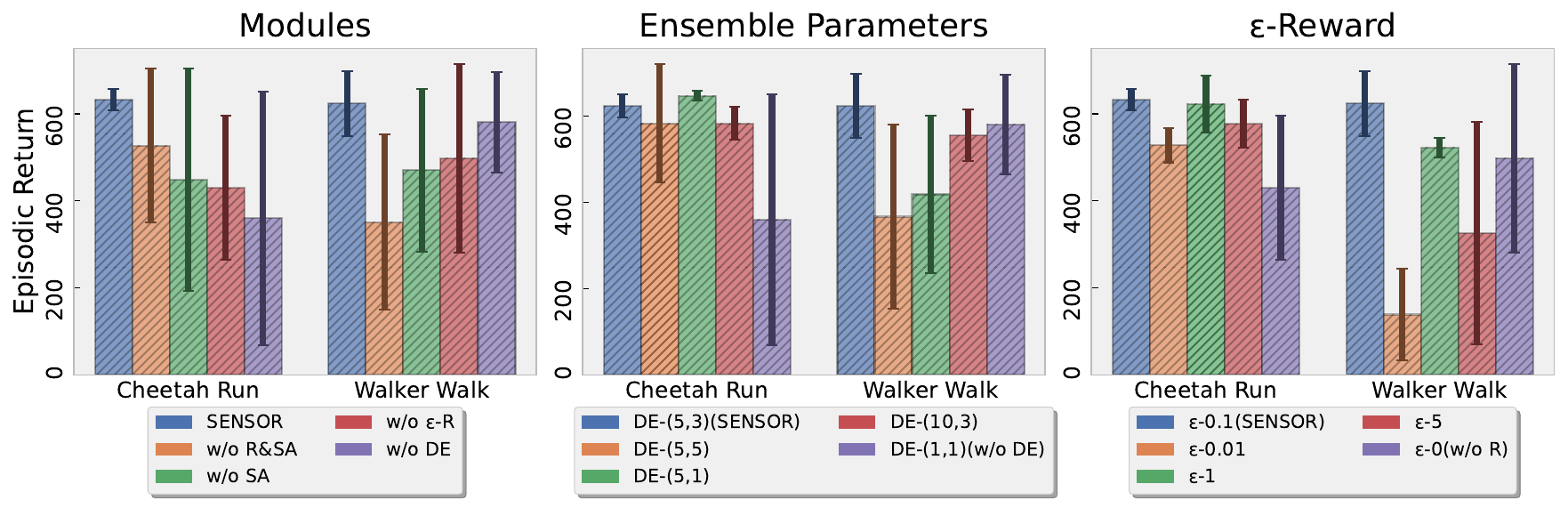}
    \caption{Ablation studies on DMC at step 500K. \textbf{Left}: Ablation on the design of separate actors and the design of $\epsilon$-reward. \textbf{Middle}: Results of different settings of discriminator ensemble parameters. \textbf{Right}: Results of different settings of $\epsilon$-reward.} 
    \label{f-ablation}
\end{figure}

\subsection{How is the contribution of different components of SENSOR to its final performance?}
\label{s-exp-ablation}
To evaluate the contribution of different components of SENSOR, we conduct ablation studies on the separate actors (SA), the design of the discriminator ensemble (DE), and the $\epsilon$-reward (R). First, we validate the effectiveness of SA and R-$0.1$ ($\epsilon_0=0.1$) via ablation experiments with the discriminator's parameters fixed at DE-$(5,3)$. Next, we will explore in detail the impact of different hyperparameter settings of the discriminator ensemble and $\epsilon$-reward on the experimental performance. We conduct these experiments on two environments with 500K steps and 3 random seeds for each experiment, and we set the agent's initial viewpoint as \texttt{d5a60e-20} for Walker Walk and \texttt{d1a120e-30} for Cheetah Run. Full ablation can be found in \cref{s-full-ablation}.

\textbf{Ablation on designed modules.}\,\, In terms of actor design, we compare the difference between the design of separate actors $a_t=\text{concat}(\pi^z(G_z(s_t)),\pi^c(G_c(s_t)))$ used in SENSOR and the single-actor design $a_t=\pi(s_t)$ used in V-MAIL. We compare the following two different types of reward design: the $\epsilon$-reward technique proposed in \cref{s-reward} and a naive design that simply averages the outputs of $N_2$ discriminators. As shown in \cref{f-ablation}-(1), removing SA results in a significant performance drop, indicating that designing two separate actors to control the agent and the camera can facilitate viewpoint alignment and improve the final performance. After eliminating $\epsilon$-reward, we observe an increase in variance and a decrease in performance, revealing that controlling the exploration-exploitation trade-off via $\epsilon$ allows for better viewpoint exploration in the early stage and more stable performance in the later stage. The performance of w/o R\&SA surpasses w/o SA and w/o R in Cheetah Run, indicating that both components, SA and R, are crucial and indispensable.

\begin{table}[t]
\renewcommand{\arraystretch}{1.25}
\vskip 0.15in
\begin{center}
\begin{small}
\begin{tabular}{c|c|c}
\toprule
\textbf{Initial Viewpoint}  &  SENSOR  &  V-MAIL  \\
\midrule
Expert  & 889.54 $\pm$ 52.68 &  \textbf{953.325 $\pm$ 10.45} \\ 
Easy  & \textbf{813.11 $\pm$ 72.10} &  111.59 $\pm$ 18.25 \\ 
Medium  &  \textbf{832.03 $\pm$ 69.89}  &  35.61 $\pm$ 12.75 \\
Hard  & \textbf{862.42 $\pm$ 64.70}  &  116.16 $\pm$ 95.82 \\
\bottomrule
\end{tabular}
\end{small}
\end{center}
\vskip -0.1in
\caption{The consistency performance of SENSOR and V-MAIL. We select one for each from four difficulty levels, which is Expert(\texttt{d3a90e-45}), Easy(\texttt{d3a80e-30}), Medium(\texttt{d4a80e-30}), and Hard(\texttt{d5a60e-20}). Each method is trained with three seeds, and we report the mean (higher is better) and standard deviation of the returns.}
\label{table1}
\end{table}

\textbf{Different settings of discriminator ensemble.}\,\, SENSOR designs $N_1$ discriminators and randomly select $N_2$ of them to update while updating discriminators. We investigate the effects of different settings of DE on experimental performance. As shown in \cref{f-ablation}-(2), we find that setting the value of $N_2$ to around half of $N_1$, such as DE-$(5,3)$, can lead to optimal performance. The smaller $N_2$, such as DE-$(5,1)$, may cause excessive randomness, and make discriminators updating insufficiently which will lead to underfitting. The larger $N_2$, such as DE-$(5,5)$, disables the drop operation and eliminates randomness, which can harm the robustness of the discriminators. Additionally, more discriminators do not necessarily lead to a better performance. Removing DE results in a considerable drop of performance and training stability, showing that the output of a single discriminator is unreliable under the active vision framework. The relative performance of experiments differs across the two environments, demonstrating that hyperparameters related to the reward have a significant impact on the performance thus require fine-tuning.

\begin{figure}[t] 
    \centering
    \includegraphics[width=0.48\textwidth]{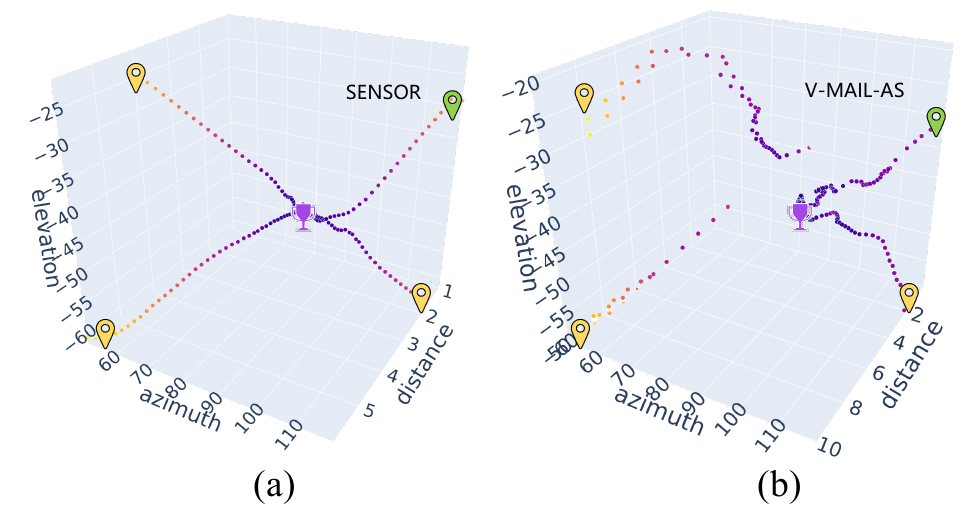}
    \caption{The 3D plot on locomotion of viewpoints with different initialization. \textbf{Left}: We train SENSOR in Cheetah Run for 1M steps with viewpoint initialization of one hard case(\texttt{d1a120e-30}) (marked with a green locator), and test on it along with three different hard initialization(\texttt{d6a50e-60}, \texttt{d5a60e-20}, \texttt{d2a120e-60}), corresponding to the four corners of the figure. The expert viewpoint is marked with a purple trophy. \textbf{Right}: The algorithm is switched to V-MAIL-AS, and all other settings are the same as on the left.}
    \label{f-robust}
\end{figure}

\textbf{Different settings of $\epsilon$-reward.}\,\, To evaluate the effects of different settings of the $\epsilon$-reward on final experimental performance, we conduct experiments on both Cheetah Run and Walker Walk. We find that a reasonable range for the value of $\epsilon_0$ is $[0.1, 1]$, which demonstrates the effectiveness of $\epsilon$-reward that encourages exploration in the early stages and focuses on exploitation in the later stages. Both excessively large or small value of $\epsilon_0$ are detrimental to model learning. A large $\epsilon_0$ leads to excessive exploration, preventing the perspective from stably converging. On the other hand, a small $\epsilon_0$ results in insufficient exploration, hindering the model from escaping local optima and finding the expert perspective. The performance variance across different $\epsilon_0$ settings is smaller on Cheetah Run than on Walker Walk. A possible explanation is the difference between two dynamics. The agent in Cheetah Run is less likely to fall compared to Walker Walk, leading to more stable performance.

\subsection{Is the performance of SENSOR consistent and robust to changes in perspective?}
\label{s-exp-consistency}

\textbf{Consistency.}\,\, To test the consistency over different initial viewpoints, we conduct several experiments on SENSOR and V-MAIL with varying difficulty measured according to the gap between the initial and expert viewpoints, and report the evaluation results in Table 1. SENSOR outstrips V-MAIL in all difficulty levels with significantly lower variance, revealing its strong consistency and adaptability across different initial perspectives. V-MAIL only performed well on the Expert level and exhibited high variance in performance across different perspectives. SENSOR and V-MAIL represent active vision and domain adaptation approaches, respectively. The experimental results above demonstrate that the active vision framework exhibits significant consistency and stability, making it more suitable for solving third-person imitation learning problems.

\textbf{Robustness.}\,\, To be applicable to a wider range of real-world scenarios, we need to ensure the robustness of our method to different viewpoints. We train SENSOR and V-MAIL-AS in Cheetah Run for 1M steps with viewpoint initialization as one hard perspective, and test it on four hard viewpoint initialization to test the learned policy. We show the locomotion of viewpoints of test trajectories in \cref{f-robust}. The policy learned on one viewpoint initialization by SENSOR can converge to the expert viewpoint uniformly and stably on all four viewpoints, demonstrating that SENSOR learns the correct pattern of adjusting the viewpoint, rather than simply imitating the expert demonstrations. In contrast, the viewpoint locomotion of V-MAIL-AS is unstable, and the agent cannot achieve the expert viewpoint under distant viewpoint initializations.

\begin{figure}[t] 
    \centering
    \includegraphics[width=0.48\textwidth]{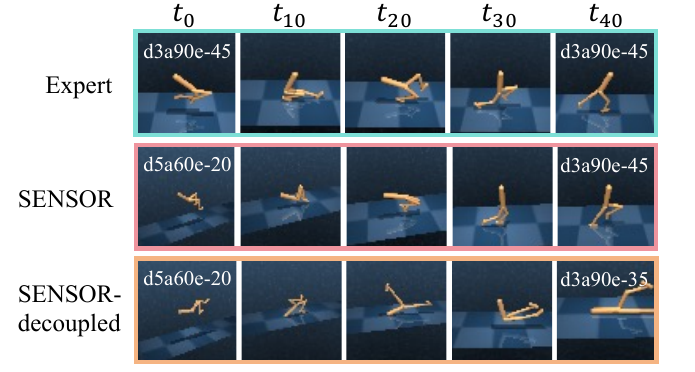}
    \caption{Examples of rendered images selected from complete trajectories. The three rows from top to bottom correspond to the expert policy, and the trained policy from SENOR and SENSOR-decoupled respectively. The images in each column correspond to the same timestep.}
    \label{f-sensor-vs-decoupled}
\end{figure}

\subsection{How does decoupled-dynamics method perform compared to SENSOR?}
\label{decouple-poor}
As shown in the aforementioned experiments, SENSOR exhibits consistent performance across different perspectives and can converge to the expert perspective quickly and stably, demonstrating that incorporating active vision in imitation learning is feasible. SENSOR only decouples motor and sensor at the policy level, but not at the world model level. Going further, \textit{can the motor and the sensor be completely decoupled?} To verify this idea, we design \textbf{SENSOR-decoupled}, a variant of SENSOR, which uses two parallel RSSMs to capture the dynamics of motor state $z$ and the sensor state $c$ independently and a joint observation model $\hat{o}\sim p_{\psi_2}(\hat{o}|z,c)$. 

To answer the problem above, we train both SENSOR and SENSOR-decoupled with initial viewpoint of \texttt{d5a60e-20}, and report several selected images from test trajectories in \cref{f-sensor-vs-decoupled}. SENSOR is well aligned to the expert perspective as well as the motor policy, demonstrating the effectiveness of active vision framework. The distance and azimuth of SENSOR-decoupled are aligned to the expert's, while there is a misalignment of elevation, which leads to the failure of motor policy imitating. The camera always stares at the agent's torso, thus the elevation will be affected when the agent falls, indicating that motor and sensor cannot be completely decoupled. To formally provide an explanation for the failure of SENSOR-decoupled, we derive an upper bound of SENSOR-decoupled's loss function in the \cref{poor} and then analyze in detail.

\section{Conclusion}
In many realistic Imitation Learning scenarios, there exist inconsistencies in the perspectives between the expert and the imitator, which will affect the effectiveness of Imitation Learning algorithms. Some existing third-person Imitation Learning methods based on domain alignment introduce extra cost and are difficult to deal with situations where the viewpoint gap is too large. To address these challenges, we introduce active sensoring in the visual RL setting and propose a model-based SENSory imitatOR (SENSOR), which jointly learns a world model to capture the dynamics of latent states, a sensor policy to control the camera, and a motor policy to control the agent. Experiments on DMC with different perspectives demonstrate that SENSOR can effectively adjust its perspective to match the expert's. Further experiments demonstrate the effectiveness of the various modules of our approach, as well as the consistency and robustness of the algorithm.

\textbf{Limitations and Future Work}.\,\, When the expert perspective remains constant, our experiments demonstrate that even in the hard cases, SENSOR achieves superior performance over baseline methods. However, in most real-world tasks, the expert perspective is likely to change (e.g., in maze navigation and autonomous driving). Future work can explore the extension of SENSOR to handle variable perspective imitation learning.

\section*{Societal Impact}

This paper presents work that aims to advance the field of Machine Learning. There are many potential societal consequences of our work, none of which we feel must be specifically highlighted here.


\bibliography{example_paper}
\bibliographystyle{icml2024}

\newpage
\appendix
\onecolumn
\section{Algorithm}
\begin{algorithm}
    \renewcommand{\algorithmicrequire}{\textbf{Input:}}
    \renewcommand{\algorithmicensure}{\textbf{Output:}}
    \caption{SENSOR:\,\,Imitate third-person expert's Behaviors via Active Sensoring}
    \label{alg1}
    \begin{algorithmic}[1]
    \REQUIRE Expert buffer $\mathcal{B}_E$, agent buffer $\mathcal{B}_{\pi}$
    \STATE Randomly initialize world model $\left\{q_{\omega},p_{\theta},p_{\zeta}\right\}$, motor and sensor policies $\pi=\left\{\pi^z,\pi^c\right\}$, encoders $\{G_z,G_c\}$, discriminator $D_{\psi}$ and critic $V$ 
    \FOR{epoch $n=1:N$}
    \FOR{timestep $t=1:T$}
    \STATE\textcolor{red}{// Collect training data}
    \STATE Infer latent state $s_t\sim q_{\omega}(s_t|s_{t-1},a_{t-1},o_t)$
    \STATE Get motor state and sensor state $z_t=G_z(s_t)$, $c_t=G_c(s_t)$
    \STATE Sample motor and sensor actions $a_t^z\sim \pi^z(a_t^z|z_t)$, $a_t^c\sim \pi^c(a_t^c|c_t)$
    \STATE Let agent takes $a^z_t$ and camera takes $a^c_t$, then get next observation $o_{t+1}=\text{env.step}(a^z_t,a^c_t)$
    \STATE $\mathcal{B}_{\pi} = \mathcal{B}_{\pi} \cup \{o_t, a_t=\text{concat}(a^z_t,a^c_t)\}$
    \ENDFOR
    \FOR{iteration $m=1:M$}
    \STATE\textcolor{red}{// Update world model}
    \STATE Sample batch $\{o_{1:H},a_{1:H-1}\}$ from $\mathcal{B}_E\cup\mathcal{B}_{\pi}$
    \STATE Rollout with $q_{\omega}$ and $p_{\theta}$ to get prior $B_{\text{prior}}=\{\hat{s}_{1:H}\}$ and posterior $B_{\text{post1}}=\{s_{1:H}\}$
    \STATE Use $B_{\text{post1}}$ to compute reconstruction loss $\mathcal{L}_r^t=-\ln p_{\zeta}(o_i|s_i)$
    \STATE Use $B_{\text{post1}}$ and $B_{\text{prior}}$ to compute consistency loss $\mathcal{L}_c^t$
    \STATE Update $\{q_{\omega},p_{\theta},p_{\zeta}\}$ using \cref{e3} with $B_{\text{prior}}$ and $B_{\text{post1}}$
    
    \STATE\textcolor{red}{// Update discriminator}
    \STATE Use $\pi$, $p_{\theta}$ and $\{G_z,G_c\}$ to get imaginary trajectory $B_{\text{imag}}=\{\hat{s}^{\pi}_{1:H},a^{\pi}_{1:H-1}\}$
    \STATE Sample batch $\{o^E_{1:H},a^E_{1:H-1}\}$ from $\mathcal{B}_E$ and use $q_{\omega}$ to compute posterior $B_{\text{post2}}=\{s^E_{1:H},a^E_{1:H-1}\}$
    \STATE Update $D_{\psi}$ using \cref{e-gail} with $B_{\text{imag}}$ and $B_{\text{post2}}$

    \STATE\textcolor{red}{// Update critic, policies and encoders}
    \STATE Get rewards $\{r_{1:H}\}$ using $D_{\psi}(\hat{s}^{\pi}_t,a^{\pi}_t)$ and get state-values $\{v_{1:H}\}$ using critic $V(\hat{s}^{\pi}_t)$ with $B_{\text{imag}}$
    \STATE Compute critic target $v^K_t = v^K(\hat{s}^{\pi}_t) = \sum^{t+K-1}_{\tau=t}\gamma^{\tau-t}\log r_{\tau} + \gamma^k v_{t+K}$
    \STATE Update $V$ by minimizing $J(V)=\sum_{t=1}^H(V(\hat{s}^{\pi}_t)-v^K_t)^2$
    \STATE Update policies and encoders by $\max_{(\pi^z,\pi^c,G_z,G_c)}\sum_{t=1}^H(v^K_t)$
    \ENDFOR
    \ENDFOR
\end{algorithmic}  
\end{algorithm}

\section{Proof}

\subsection{Explanation of SENSOR-decoupled's poor performance}
\label{poor}
SENSOR has shown that decoupling the policy into motor and sensor parts can improve convergence speed and final performance. Going further, we want to validate the idea of whether motor and sensor states can be completely decoupled. To test this idea, we design SENSOR-decoupled that uses two parallel RSSMs to capture the dynamics of motor state $z$ and the sensor state $c$ independently and designs a different observation model $\hat{o}\sim p_{\omega_2}(\hat{o}|z,c)$. Through experiments in \cref{decouple-poor}, we find that compared to SENSOR, SENSOR-decoupled is not stable enough with high variance. Then we attempt to derive an upper bound on the SENSOR-decoupled's loss, which can provide a reasonable explanation for its poor performance.
\begin{proposition}
\label{prop1}
 (Divergence in latent space) Given POMDP $\mathcal{M}$, history $h_t=(o_{\leq t},a_{<t})$ and latent representation of history $\hat{s}_t=q(h_t)$. Let $s_t\sim P(s_t|h_t)\approx P(s_t|\hat{s}_t)$, $a^z\sim \pi^z$ and $a^c\sim \pi^c$. $\mathbb{D}_f$ means f-divergence. Then
 \begin{align*}
    \mathbb{D}_f(\rho^{\pi}_{\mathcal{M}}(\hat{s},a) \Vert \rho^E_{\mathcal{M}}(\hat{s},a)) \leq \frac{1}{2} (&\mathbb{D}_f(\rho^{\pi}_{\mathcal{M}}(z,a^z) \Vert \rho^{E}_{\mathcal{M}}(z,a^z)) \\
    + &\mathbb{D}_f(\rho^{\pi}_{\mathcal{M}}(c,a^c) \Vert \rho^{E}_{\mathcal{M}}(c,a^c)))
\end{align*}
\end{proposition}
\begin{proof}[Proof of \cref{prop1}]
Imitation learning can be transformed into divergence minimization in the latent belief state representation \cite{rafailov2021visual}, and we get result 
\begin{equation*}
    \mathbb{D}_f(\rho^{\pi}_{\mathcal{M}}(o,a) \Vert \rho^E_{\mathcal{M}}(o,a)) \leq \mathbb{D}_f(\rho^{\pi}_{\mathcal{M}}(s,a) \Vert \rho^E_{\mathcal{M}}(s,a)) \leq \mathbb{D}_f(\rho^{\pi}_{\mathcal{M}}(\hat{s},a) \Vert \rho^E_{\mathcal{M}}(\hat{s},a))
\end{equation*}
Based on the definition of f-divergence $\mathbb{D}_f(p\Vert q)=\int q(x)f\left( \frac{p(x)}{q(x)} \right)\text{d}x$ and expectation, we can get
\begin{equation*}
    \mathbb{D}_f(\rho^{\pi}_{\mathcal{M}}(\hat{s},a) \Vert \rho^{E}_{\mathcal{M}}(\hat{s},a)) = \underbrace{\mathbb{E}_{\hat{s},a\sim \rho^E_{\mathcal{M}}(\hat{s},a)} \left[f\left( \frac{\rho^{\pi}_{\mathcal{M}}(\hat{s},a)}{\rho^{E}_{\mathcal{M}}(\hat{s},a)} \right)\right]}_{(a)}
\end{equation*}
In SENSOR-decouple's setting, state $\hat{s}$ and action $a$ can be divided into the motor part and the sensor part. Together with Bayes' theorem, we have
\begin{align}
    \label{e8}
    (a) &= \mathbb{E}_{z,c,a^z,a^c \sim \rho^E_{\mathcal{M}}(\hat{s},a)} \left[f\left( \frac{\rho^{\pi}_{\mathcal{M}}(z,a^z)}{\rho^{E}_{\mathcal{M}}(z,a^z)} \cdot \frac{p(c,a^c|z,a^z)}{p(c,a^c|z,a^z)} \right)\right] 
    \\
    \label{e9}
    &= \mathbb{E}_{z,c,a^z,a^c \sim \rho^E_{\mathcal{M}}(\hat{s},a)} \left[f\left( \frac{\rho^{\pi}_{\mathcal{M}}(z,a^z)}{\rho^{E}_{\mathcal{M}}(z,a^z)} \cdot \frac{p(c|z)}{p(c|z)} \cdot \frac{\pi^c(a^c|c)}{\pi^c_E(a^c|c)} \right)\right]
    \\
    \label{e10}
    &= \mathbb{E}_{z,c,a^z,a^c \sim \rho^E_{\mathcal{M}}(\hat{s},a)} \left[f\left( \frac{\rho^{\pi}_{\mathcal{M}}(z,a^z)}{\rho^{E}_{\mathcal{M}}(z,a^z)} \cdot \pi^c(a^c|c) \right)\right]
    \\
    \label{e11}
    &\leq \mathbb{E}_{z,c,a^z,a^c \sim \rho^E_{\mathcal{M}}(\hat{s},a)} \left[f\left( \frac{\rho^{\pi}_{\mathcal{M}}(z,a^z)}{\rho^{E}_{\mathcal{M}}(z,a^z)} \right)\right] 
    \\
    \label{e12}
    &= \mathbb{E}_{z,c,a^z,a^c \sim \rho^E_{\mathcal{M}}(\hat{s},a)} \left[f\left( \frac{\rho^{\pi}_{\mathcal{M}}(c,a^c)}{\rho^{E}_{\mathcal{M}}(c,a^c)} \right)\right] 
    \\
    \label{e13}
    &= \frac{1}{2} \mathbb{E}_{z,c,a^z,a^c \sim \rho^E_{\mathcal{M}}(\hat{s},a)} \left[f\left( \frac{\rho^{\pi}_{\mathcal{M}}(z,a^z)}{\rho^{E}_{\mathcal{M}}(z,a^z)} \right) + f\left( \frac{\rho^{\pi}_{\mathcal{M}}(c,a^c)}{\rho^{E}_{\mathcal{M}}(c,a^c)} \right)\right]
\end{align}

The samples $(c,a^c)$ are drawn from the $\mathcal{B}_E$. The expert's viewpoint $c$ is fixed and its sensor actions $a^c$ are always zero, thus $\pi_E^c(a^c|c)\equiv 1$ and \cref{e10} holds. \cref{e11} comes from an obvious assumption that $\pi^c(a^c|c)\leq 1$. In SENSOR-decoupled's setting, motor state $z$ and sensor state $c$ are somewhat equivalent and can be interchanged, so \cref{e11} and \cref{e12} are completely equivalent and can be directly derived into \cref{e13}. Then \cref{e13} can be reduced to f-divergence
\begin{align}
    \text{\cref{e13}} &\leq \frac{1}{2}\mathbb{E}_{z,a^z \sim \rho^E_{\mathcal{M}}(z,a^z)} \left[f\left( \frac{\rho^{\pi}_{\mathcal{M}}(z,a^z)}{\rho^{E}_{\mathcal{M}}(z,a^z)} \right)\right] + \frac{1}{2}\mathbb{E}_{c,a^c \sim \rho^E_{\mathcal{M}}(c,a^c)} \left[f\left( \frac{\rho^{\pi}_{\mathcal{M}}(c,a^c)}{\rho^{E}_{\mathcal{M}}(c,a^c)} \right)\right] \\
    &= \frac{1}{2} (\mathbb{D}_f(\rho^{\pi}_{\mathcal{M}}(z,a^z) \Vert \rho^{E}_{\mathcal{M}}(z,a^z)) + \mathbb{D}_f(\rho^{\pi}_{\mathcal{M}}(c,a^c) \Vert \rho^{E}_{\mathcal{M}}(c,a^c)))
\end{align}
\end{proof}
\cref{prop1} shows that the gap between the agent and the expert occupancy measures will decrease with continuous optimization of the motor policy gap (see \cref{e11}) and sensor policy gap (see \cref{e12}). In principle, if both branches can be optimized well, the final performance will be good. However, this goal is difficult to achieve for the following reasons:
\begin{itemize}
    \item Compared with other world model candidates like RNNs and SSMs, RSSM is more complex and requires careful fine-tuning to achieve stable performance. SENSOR-decoupled designs two RSSMs, so it's hard to optimize naturally.
    \item In the proof of \cref{prop1}, We assume that $\pi^c(a^c|c)\leq 1$ which is quite obvious, leading to the final loss upper bound not being tight enough. It may not be possible to achieve good results based on this loose upper bound.
    \item In some environments of DMC, the camera is set to always focus on a certain part of the agent (default lookat is torso). Take Walker as an example, the camera's relative height to the floor will change while the agent is moving, thereby affecting the sensor's judgment of elevation in observation space. In this case, the motor state $z$ and the sensor state $c$ cannot be decoupled completely.
\end{itemize}

\section{Additional results}
\label{s-additional-results}

\begin{figure*}[t] 
    \centering
    \includegraphics[width=\textwidth]{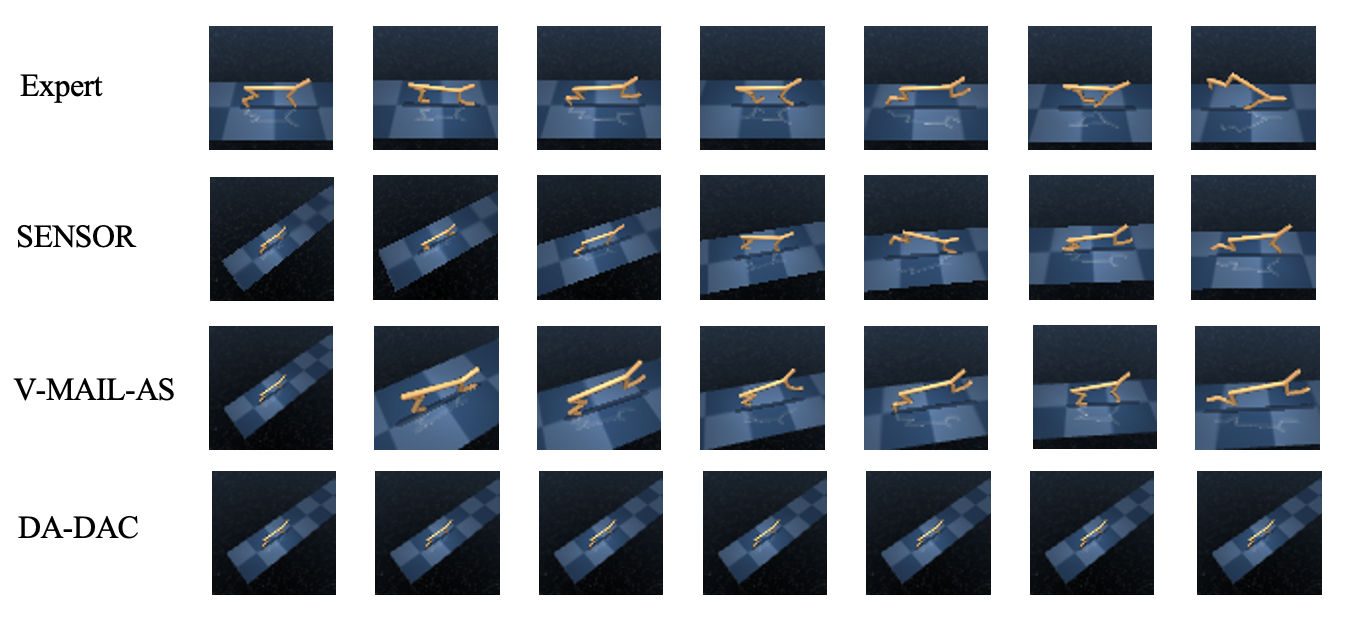}
    \caption{Visualization of observations collected by different trained agents.} 
    \label{multi-trajs}
\end{figure*}

\begin{figure*}[t] 
    \centering
    \includegraphics[width=\textwidth]{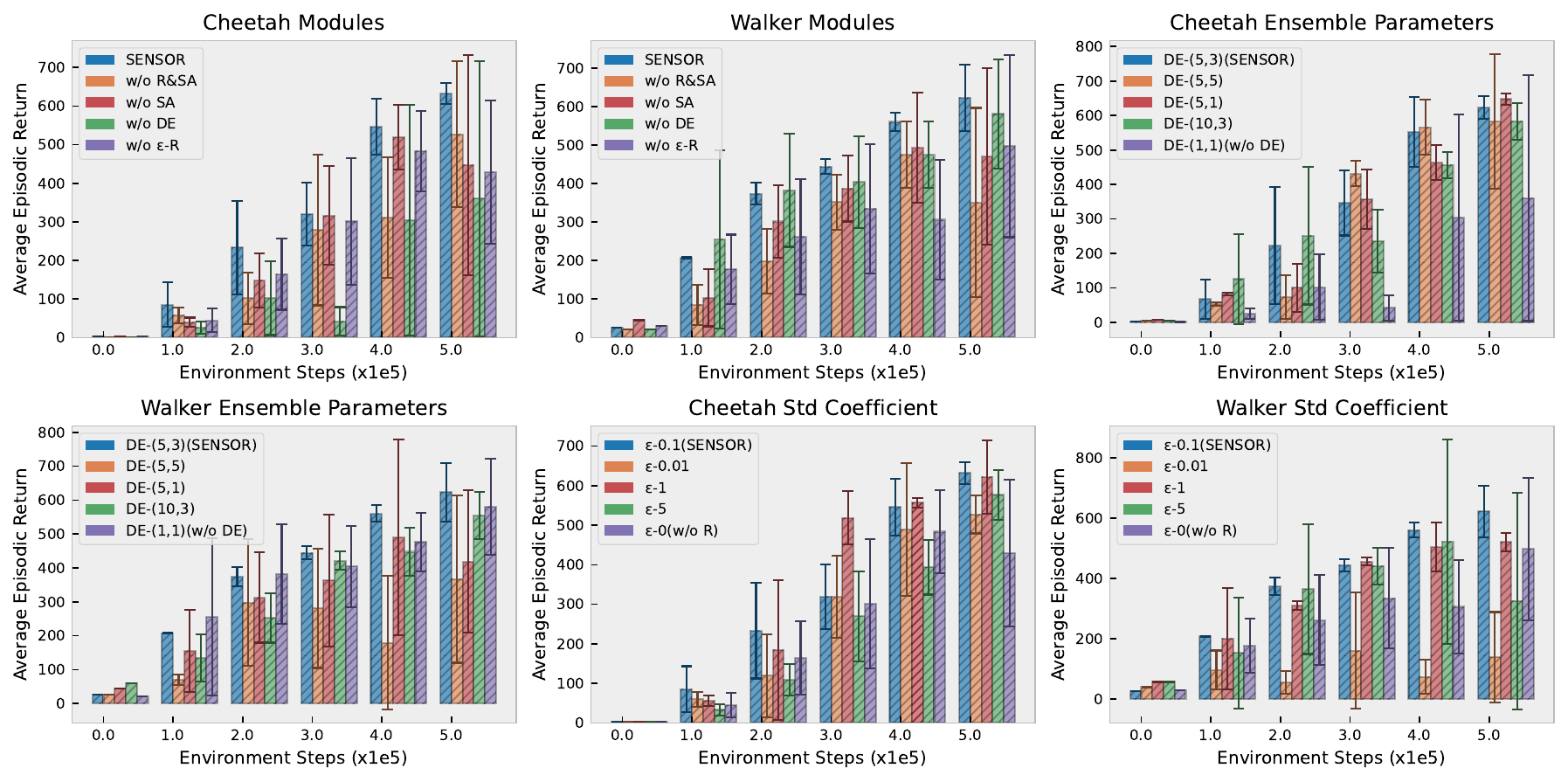}
    \caption{Ablation Results of Extra Environment Steps} 
    \label{full-ablation}
\end{figure*}

\subsection{Visualization of observational trajectories}
As shown in \cref{multi-trajs}, the first row is the expert trajectory and the second, third and fourth rows are the test trajectories generated using the strategies obtained from the training of Algorithms SENSOR, V-MAIL-AS, DA-DAC, respectively. It can be observed that SENSOR and V-MAIL-AS are more successful in mimicking the expert's perspective and show better performance in the trajectory. DA-DAC does not mimic both the expert perspective and the motor strategy well enough.

\subsection{More Ablation Results}
\label{s-full-ablation}

In \cref{full-ablation}, we show the full ablation results of full training process. SENSOR is able to perform competitively enough among different time steps. We find that ablation does not perform consistently on different environments, but SENSOR is able to do a good job on different environments. Ensemble parameters performs good when $N_2 / N_1$ is near 0.5, which shares similar results as in dropout\cite{dropout}. Cheetah is more robust to parameter changes in std, since the cheetah stand at start point while the Walker is easier to fall down, making it sensitive to changes in std coefficient which controls exploration.

\section{Implementation Details}
\subsection{Detailed Description about Baselines}
\textbf{Behavior Cloning.}\,\, Behavior cloning(BC) is a basic method in imitation learning where an agent learns to mimic the behavior of an expert by observing and imitating its actions in a given environment. It uses a supervised way to directly map observation into actions. We implement BC using the same convolution network as in SENSOR, and come after 3 linear layers using tanh as activation function.

\textbf{V-MAIL(-AS).}\,\, Visual Adversarial Imitation Learning using Variational Models(V-MAIL) is a method that addresses the challenges of learning behaviors through deep reinforcement learning by using visual demonstrations of desired behaviors. The algorithm learns successful visuomotor policies in a sample-efficient manner, achieves higher asymptotic performance, and can learn new tasks from visual demonstrations without additional environment reward signals. However, the method has limitations such as vulnerability to adversarial visual perturbations and potential difficulty in representing complex scenes with multiple cluttered or deformable objects. V-MAIL trains a variational latent-space dynamics model and a discriminator that provides a learning reward signal by distinguishing latent rollouts of the agent from the expert. We implement a PyTorch \cite{paszke2019pytorch} version of V-MAIL following the official implementation. V-MAIL-AS extends the action space of policy and keeps others the same.

\textbf{DA-DAC-AS.}\,\, DA-DAC is a data-augmentation version of Discriminator Actor-Critic(DAC). The DAC algorithm is an approach that addresses the issues of reward bias and sample inefficiency in Adversarial Imitation Learning. It uses off-policy Reinforcement Learning to reduce the number of interactions with the environment and has an unbiased reward function, making it applicable to a wide range of problems. We use an unofficial version and modifies the environment and the related dimensions, and implement DA-DAC-AS based on it.

\textbf{TPIL.}\,\, TPIL is a method for training agents in reinforcement learning to achieve goals in complex environments by observing third-person demonstrations from a different viewpoint.
Unlike traditional first-person imitation learning, TPIL does not require the agent to be provided with first-person demonstrations, which can be challenging to collect.
The TPIL method utilizes recent advances in domain confusion to generate domain agnostic features during the training process.TPIL is built on AIL framework, and utilize another classifier to distinguish between expert domain and agent domain. We implement TPIL based on an imitation repo.

\textbf{DisentanGAIL.}\,\, DisentanGAIL is an algorithm that enables autonomous agents to learn from high dimensional observations of an expert performing a task. It uses adversarial learning with a latent representation inside the discriminator network and regularizes the latent representation through mutual information constraints. This encourages the learning of features that encode information about the completion levels of the task being demonstrated. The algorithm allows for successful imitation learning while disregarding differences between the expert's and the agent's domains. DisentanGAIL divides the information of state embedding into two parts: domain information and goal-completion information. We use the official implementation and modifies the environment and the related dimensions.

\subsection{Experiments Details}

\textbf{Environment.}\,\, We do experiments on MuJoCo \cite{todorov2012mujoco} and DeepMind Control(DMC) \cite{tassa2018deepmind} Suite environment. DMC includes a collection of continuous control tasks that researchers can use to evaluate and benchmark their reinforcement learning algorithms. We use Walker Walk and Cheetah Run, and use Movable Camera to generate different views under the same motor state.

\textbf{Expert Datasets.}\,\, In MuJoCo we use SAC \cite{sac} to collect expert data and use a random policy to get prior data used in DisentanGAIL. We use Dreamer to train 500k steps and use the trained agent to collect 10 trajectories for each environments. The average return of Cheetah and Walker is around 900, and shares low variance.


\textbf{SENSOR Details}\,\, We do our experiments on 12 RTX3090s. One experiment takes about 50 GPU hours. 

\section{Hyper-parameters}
\begin{table}[htbp]
\begin{center}
\begin{small}
\setlength\tabcolsep{1.2pt}
\scalebox{0.95}{\begin{tabular}{ll}
\toprule
\textbf{Hyperparameter} & \textbf{Value}  \\
\midrule
Deterministic size & $200$ \\
Stochastic size & $30$ \\
Embedding size & $1024$ \\
Sequence length $T$ & $50$  \\ 
Batch size & $50$\\
Imagine horizon $H$ & $15$ \\
Optimizer & Adamw \\
Learning rate & $6\times 10^{-4}$ world model, $8\times 10^{-5}$ discriminator \\&$8\times 10^{-5}$ actor, $8\times 10^{-5}$ critic\\
Encoder channels & $32, 64, 64$ \\
Ensemble parameters & select 3 out of 5\\
Std coefficient & 0.1 \\
Repeat time of $a$ & $2$ \\
Grad norm clip & $100$ \\
Weight decay & $0.001$ for Cheetah Run\\
& $0.0$ for Walker Walk\\
\bottomrule
\end{tabular}}
\end{small}
\end{center}
\caption{Hyperparameters of SENSOR for all experiments.}
\label{tab:atari_hyperparameters}
\end{table}


\end{document}